\documentclass{article}

\usepackage{amsmath}
\usepackage{epstopdf}
\usepackage{mathtools}
\usepackage{amsfonts}
\usepackage{amsthm}

\theoremstyle{definition}
\newtheorem{mydef}{Definition}
\newtheorem{thm}[mydef]{Theorem}
\newtheorem{lem}[mydef]{Lemma}
\newtheorem{prop}[mydef]{Proposition}
\newtheorem{cor}[mydef]{Corollary}

\newcommand\R{\mathbb{R}}

\newcommand\N{\mathbb{N}}

\newcommand\B{\mathbb{B}}
\newcommand\E{\mathbb{E}}
\newcommand\eps{\epsilon}
\newcommand\vareps{\varepsilon}

\newcommand\argmin{\text{argmin}}

\newcommand\Sf{S_{\text{Fast}}}
\newcommand\Ktc{Kt\text{-cost}}

\usepackage{natbib}

\usepackage{microtype}

\makeatletter
\newcommand{\oset}[3][0ex]{%
  \mathrel{\mathop{#3}\limits^{
    \vbox to#1{\kern-2\ex@
    \hbox{$\scriptstyle#2$}\vss}}}}
\makeatother

\def\timesleq{\oset[.2ex]{\times}{\leq}} 
\def\timesgeq{\oset[.2ex]{\times}{\geq}} 
\def\timeseq{\oset[.3ex]{\times}{=}}     

\allowdisplaybreaks[1]

\usepackage{hyperref}

\title{Loss Bounds and Time Complexity for Speed Priors}
\author{Daniel Filan, Marcus Hutter, Jan Leike}
\date{}

\begin{document}

\bibliographystyle{plainnat}
\bibpunct{(}{)}{;}{a}{,}{,}

\maketitle

\begin{abstract}

  This paper establishes for the first time
the predictive performance of speed priors and their computational complexity.
A speed prior is essentially a probability distribution that puts low probability on strings that are not efficiently computable.
We propose a variant to the original speed prior~\citep{schmidhuber2002speed},
and show that our prior can predict sequences drawn from probability measures
that are estimable in polynomial time.
Our speed prior is computable in doubly-exponential time,
but not in polynomial time.
On a polynomial time computable sequence
our speed prior is computable in exponential time.
We show better upper complexity bounds for Schmidhuber's speed prior
under the same conditions,
and that it predicts deterministic sequences that are computable in polynomial time;
however, we also show that it is not computable in polynomial time,
and the question of its predictive properties for stochastic sequences remains open.
\end{abstract}
\tableofcontents

\section{Introduction}
\label{sec:introduction}

We consider the general problem of sequence prediction, where a sequence of symbols $x_1, x_2, \dotsc, x_{t-1}$ is drawn from an unknown computable distribution $\mu$, and the task is to predict the next symbol $x_t$.
If $\mu$ belongs to some known countable class of distributions,
then a Bayesian mixture over the class leads to
good loss bounds:
the expected loss is at most $L + O(\sqrt{L})$
where $L$ is the loss of the informed predictor that knows $\mu$~\citep[Thm.\ 3.48]{marcus2005universal}.
These bounds are known to be tight.

Solomonoff's theory of inductive inference handles the most general case where all we know about $\mu$ is that it is computable~\citep{solomonoff1964formal1, solomonoff1978complexity}.
The Solomonoff prior $M$ assigns to a string
$x$ the probability that a universal Turing machine prints something
starting with $x$ when fed with fair coin flips.
Equivalently,
the distribution $M$ can be seen as a Bayesian mixture
that weighs each distribution according to their Kolmogorov complexity%
~\citep{wood2013non},
assigning higher a priori probability to simpler hypotheses%
~\citep{hutter2007universalbayes}.
However, $M$ is incomputable~\citep{LH15computability2},
which has thus far limited its application.

Schmidhuber has proposed a computable alternative to $M$ which discounts strings that are not efficiently computable~\citep{schmidhuber2002speed}. This distribution is called the \emph{speed prior} because asymptotically only the computationally fastest distributions that explain the data contribute to the mixture. However, no loss bounds for Schmidhuber's prior, which we write as $\Sf$, are known except in the case where the data are drawn from a prior like Schmidhuber's.

We introduce a prior $S_{Kt}$ that is related to both $\Sf$ and $M$,
and establish in Section~\ref{sec:speed-priors} that
it is also a speed prior in Schmidhuber's sense.
Our first main contribution is a bound on the loss incurred by a $S_{Kt}$-based predictor
when predicting strings drawn from a distribution that is computable in polynomial time.
This is proved in Section~\ref{sec:loss-bounds}.
The bounds we get are only a logarithmic factor worse than
the bounds for the Solomonoff predictor.
In particular, if the measure is deterministic and the loss function penalises errors, $S_{Kt}$-based prediction will only make a logarithmic number of errors.
Therefore, $S_{Kt}$ is able to effectively learn the generating distribution $\mu$. Our second main contribution is a proof that the same bound holds for the loss incurred by a $\Sf$-based predictor when computing a string deterministically generated in polynomial time, shown in the same section.

In Section~\ref{sec:time-complexity} we discuss
the time complexity of $S_{Kt}$ and $\Sf$.
We show that $S_{\text{Fast}}$ is computable in exponential time while $S_{Kt}$ is computable in doubly-exponential time, but not in polynomial time, limiting its practical applicability. However, we also show that if we are predicting a sequence that is computable in polynomial time, it only takes polynomial time to compute $\Sf$ and exponential time to compute $S_{Kt}$.

Although the results of this paper are theoretical and the algorithms im\-practical-seeming, related ideas from the field of algorithmic information theory have been approximated and put into practice. Examples include the Universal Similarity Metric's use in clustering~\citep{cilibrasi2005clustering}, Solomonoff-based reinforcement learning~\citep{veness2011monte}, and the Levin search-inspired Optimal Ordered Problem Solver~\citep{schmidhuber2004optimal}. However, using the theory to devise practical applications is a non-trivial task that we leave for future work.

\section{Preliminaries}
\label{sec:preliminaries}

\subsection{Setup and notation}
\label{sec:setup-notation}

Throughout this paper, we use monotone Turing machines with a binary alphabet $\B = \{0,1\}$, although all results generalise to arbitrary finite alphabets. A monotone machine is one with a unidirectional read-only input tape where the head can only move one way, a unidirectional write-only output tape where the head can only move one way, and some bidirectional work tapes. We say that a monotone machine $T$ computes string $x$ given program $p$ if the machine prints $x$ after reading all of $p$ but no more, and write $p \xrightarrow{T} x$~\citep[Def.\ 4.5.2]{livitanyi}. Some of these machines are universal Turing machines, or `UTM's. A UTM can simulate all other machines, so that the output of $U$ given input $I(T)p$ (where $I(T)$ is a prefix-free coding\footnote{A coding such that for no two different machines $T$ and $T'$ is $I(T)$ a prefix of $I(T')$.} of a Turing machine $T$) is the same as the output of $T$ given input $p$. Furthermore, we may assume this simulation occurs with only polynomial time overhead. In this paper, we fix a `reference' UTM $U$, and whenever a function $f(T,\dotsc)$ takes an argument $T$ that is a Turing machine, we will often write $f(\dotsc)$, where we set $T$ to be the reference UTM.

Our notation is fairly standard, with a few exceptions. If $p \xrightarrow{U} x$, then we simply write $p \rightarrow x$. We write $f(n) \timesleq g(n)$ if $f(n) = O(g(n))$, and $f(n) \timeseq g(n)$ if $f(n) \timesleq g(n)$ and $g(n) \timesleq f(n)$. Also, if $x$ is some string, we denote the length of $x$ by $|x|$. We write the set of finite binary strings as $\B^*$, the set of infinite binary sequences as $\B^\infty$, an element of $\B^\infty$ as $x_{1:\infty}$, the $n$\textsuperscript{th} symbol of a string $x$ or $x_{1:\infty}$ as $x_n$, and the first $n$ symbols of any string $x$ or $x_{1:\infty}$ as $x_{1:n}$. $\# A$ is the cardinality of set $A$. Finally, we write $x \sqsubseteq y$ if string $x$ is a prefix of string $y$, and $x \sqsubset y$ if $x$ is a proper prefix of $y$.

\subsection{$\Sf$ and $M$}
\label{sec:m-sf}

To define $\Sf$, we first need to define the \textsc{fast} algorithm (called \textsc{search} by \cite{livitanyi}, Ch. 7.5) after which it is named. This algorithm performs \textsc{phase} $i$ for each $i \in \N$, whereby $2^{i - |p|}$ instructions of all programs satisfying $|p| \leq i$ are executed as they would be on $U$, and the outputs are printed sequentially, separated by blanks. If string $x$ is computed by program $p$ in \textsc{phase} $i$, then we write $p \rightarrow_i x$. Then, $S_{\text{Fast}}$ is defined as
\begin{equation}
  \label{eq:10}
  \Sf (x) := \sum_{i = 1}^\infty 2^{-i} \sum_{p \rightarrow_i x} 2^{-|p|}
\end{equation}

This algorithm is inspired by the $Kt$ complexity of a string, defined as
\begin{equation*}
  Kt(x) = \min_p \{ |p| + \log t(U,p,x)\}
\end{equation*}
where $t(U,p,x)$ is the time taken for program $p$ to compute $x$ on the UTM $U$, and if program $p$ never computes $x$, we set $t(U,p,x) := \infty$~\citep[Def.\ 7.5.1]{livitanyi}. If we define the $Kt$-cost of a computation of a string $x$ by program $p$ as the minimand of $Kt$, that is,
\begin{equation*}
  \Ktc (p,x) := |p| + \log t(p,x)
\end{equation*}
then we can see that program $p$ computes string $x$ in \textsc{phase} $i$ of \textsc{fast} iff $\Ktc (p,x) \leq i$. As such, $\Sf$ gives low probability to strings of high $Kt$ complexity.

Similarly to the above, the monotone Kolmogorov complexity of $x$ is defined as
\begin{equation*}
  Km(x) = \min_{p} \{|p| \mid  p \rightarrow x\}
\end{equation*}
If we define the minimand of $Km$ as
\begin{equation*}
  Km\text{-cost}(p,x) :=
  \begin{dcases}
    |p| & \text{if } p \rightarrow x \\
    \infty & \text{otherwise}
  \end{dcases}
\end{equation*}
then the Solomonoff prior $M(x) = \sum_{p \rightarrow x} 2^{-|p|}$ can be written as \\ $\sum_{p \rightarrow x} 2^{- Km\text{-cost}(p,x)}$. $M$ and $\Sf$ are both semimeasures, but not measures:
\begin{mydef}
  A \emph{semimeasure} is a function $\nu: \B^* \rightarrow [0, \infty)$ such that $\nu(\eps) \leq 1$ and $\nu(x) \geq \nu(x0) + \nu(x1)$ for all $x \in \B^*$. If $\nu$ satisfies these with equality, we call $\nu$ a \emph{measure}.
\end{mydef}

Semimeasures can be used for prediction:
\begin{mydef}
If $\nu$ is a semimeasure, the \emph{$\nu$-probability of $x_t$ given $x_{<t}$} is $\nu(x_t | x_{<t}) := \nu(x_{1:t}) / \nu(x_{<t})$.
\end{mydef}

\section{Speed priors}
\label{sec:speed-priors}

By analogy to $M$, we can define a variant of the Solomonoff prior that penalises strings of high $Kt$ complexity more directly than $\Sf$ does:
\begin{align}
S_{Kt}(x) &:= \sum_{p \rightarrow x} 2^{- \Ktc (p,x)} = \sum_{p \rightarrow x} \frac{ 2^{-|p|}}{t(p,x)}  \label{eq:13}
\end{align}
$S_{Kt}$ is a semimeasure, but is not a measure.

\subsection{Similar definitions for $\Sf$ and $S_{Kt}$}
\label{sec:simpl-expr-s_textf}

The definitions \eqref{eq:10} of $\Sf$ and \eqref{eq:13} of $S_{Kt}$ have been given in different forms---the first in terms of \textsc{phase}s of \textsc{fast}, and the second in terms of $Kt$-cost. In this subsection, we show that each can be rewritten in a form similar to the other's definition, which sheds light on the differences and similarities between the two.

\begin{prop}\label{prop:s_fast_simple}
\begin{equation*}
  S_\text{Fast}(x) \timeseq \sum_{p \rightarrow x} \frac{ 2^{-2|p|}}{t(p,x)}
\end{equation*}
\end{prop}

\begin{proof}
First, we note that for each program $p$ and string $x$, if $p \rightarrow_i x$, then for all $j \geq i$, $p \rightarrow_j x$. Now,
\begin{align}
\nonumber \sum_{j = i}^\infty 2^{-j} \times 2^{-|p|} &= 2 \times 2^{-i} \times 2^{-|p|} \\
\Rightarrow  \sum_{i = 1}^\infty 2^{-i} \sum_{p \rightarrow_i x} 2^{-|p|} &\timeseq \sum_{i = 1}^\infty 2^{-i} \sum_{\substack{
    p \rightarrow_i x \\
    p \not\rightarrow_{i - 1} x}} 2^{-|p|} \label{eq:proofs7}
\end{align}
since all of the contributions to $\Sf(x)$ from program $p$ in phases $j \geq i$ add up to twice the contribution from $p$ in \textsc{phase} $i$ alone.

Next, suppose $p \rightarrow_i x$. Then, by the definition of \textsc{fast},
\begin{align*}
  &\quad t(p,x) \leq 2^{i - |p|} \\
  &\Leftrightarrow \log t(p,x) \leq i - |p| \\
  &\Leftrightarrow |p| + \log t(p,x) \leq i
\end{align*}
Also, if $p \not\rightarrow_{i - 1} x$, then either $|p| > i - 1$, implying $|p| + \log t(p,x) > i - 1$, or $t(p,x) > 2^{i - 1 - |p|}$, also implying $|p| + \log t(p,x) > i - 1$. Therefore, if $p \rightarrow_i x$ and $p \not\rightarrow_{i - 1} x$, then
\begin{equation*}
  i - 1 < |p| + \log t(p,x) \leq i
\end{equation*}
implying
\begin{equation}
  -|p| - \log t(p,x) - 1 < -i \leq -|p| - \log t(p,x)\label{eq:proofs26}
\end{equation}
Subtracting $|p|$ and exponentiating yields
\begin{align*}
  \frac{2^{-2|p| - 1}}{t(p,x)} \leq 2^{-i-|p|} \leq \frac{2^{-2|p|}}{t(p,x)} 
\end{align*}
giving
\begin{equation*}
  2^{-i-|p|} \timeseq \frac{2^{-2|p|}}{t(p,x)}
\end{equation*}

Therefore,
\begin{equation}
  \sum_{i = 1}^\infty 2^{-i} \sum_{\substack{
    p \rightarrow_i x \\
    p \not\rightarrow_{i - 1}x}} 2^{-|p|} \timeseq \sum_{p \rightarrow x} \frac{1}{t(p,x)}2^{-2|p|}\label{eq:proofs2}
\end{equation}
which, together with equation \eqref{eq:proofs7}, proves the proposition.
\end{proof}

\begin{prop}
  \begin{equation*}
    S_{Kt}(x) \timeseq \sum_{i = 1}^\infty 2^{-i} \sum_{p \rightarrow_i x} 1
  \end{equation*}
\end{prop}

\begin{proof}
    Using equation \eqref{eq:proofs26}, we have that if $p \rightarrow_i x$ and $p \not\rightarrow_{i-1}x$, then
  \begin{equation*}
    \frac{2^{-|p|-1}}{t(p,x)}  \leq 2^{-i} \leq \frac{2^{-|p|}}{t(p,x)} 
  \end{equation*}
  so
  \begin{equation*}
    2^{-i} \timeseq \frac{2^{-|p|}}{t(p,x)}
  \end{equation*}
  Summing over all programs $p$ such that $p \rightarrow_i x$ and $p \not\rightarrow_{i-1}x$, we have
  \begin{align*}
 2^{-i} \sum_{\substack{p \rightarrow_i x,\\
      p\not\rightarrow_{i-1} x}} 1 \timeseq \sum_{\substack{p \rightarrow_i x,\\
      p\not\rightarrow_{i-1} x}} \frac{2^{-|p|}}{t(p,x)}
\end{align*}
Then, summing over all phases $i$, we have
  \begin{align}
 \sum_{i=1}^\infty 2^{-i} \sum_{\substack{p \rightarrow_i x,\\
      p\not\rightarrow_{i-1} x}} 1  \timeseq \sum_{p \rightarrow x} \frac{2^{-|p|}}{t(p,x)}\label{eq:proofs1}
  \end{align}
  Now, as noted in the proof of Proposition \ref{prop:s_fast_simple}, if $q \rightarrow_i x$, then $q \rightarrow_j x$ for all $j \geq i$. Similarly to the start of that proof, we note that
  \begin{equation*}
    \sum_{i = j}^\infty 2^{-j} \times 1 = 2 \times 2^{-i} \times 1
  \end{equation*}
  The left hand side is the contribution of $q$ to the sum
  \begin{equation*}
    \sum_{i = 1}^\infty 2^{-i} \sum_{p \rightarrow_i x} 1
  \end{equation*}
  and the right hand side is twice the contribution of $q$ to the sum
  \begin{equation*}
    \sum_{i = 1}^\infty 2^{-i} \sum_{\substack{p \rightarrow_i x,\\
      p\not\rightarrow_{i-1} x}} 1
  \end{equation*}
  Therefore,
  \begin{align*}
    \sum_{i = 1}^\infty 2^{-i} \sum_{p \rightarrow_i x} 1 \timeseq \sum_{i = 1}^\infty 2^{-i} \sum_{\substack{p \rightarrow_i x,\\
      p\not\rightarrow_{i-1} x}} 1
  \end{align*}
  which, together with \eqref{eq:proofs1}, proves the proposition.
\end{proof}

\subsection{$S_{Kt}$ is a speed prior}
\label{sec:s_kt-speed-prior}

Although we have defined $S_{Kt}$, we have not shown any results that indicate it deserves to be called a speed prior. Two key properties of $\Sf$ justify its description as a speed prior: firstly, that the cumulative prior probability measure of all $x$ incomputable in time $t$ is at most inversely proportional to $t$, and secondly, that if $x_{1:\infty} \in \B^\infty$, and program $p^x \in \B^*$ computes $x_{1:n}$ within at most $f(n)$ steps, then the contribution to $\Sf(x_{1:n})$ by programs that take time much longer than $f(n)$ vanishes as $n \rightarrow \infty$~\citep{schmidhuber2002speed}. In this subsection, we prove that both of these properties also hold for $S_{Kt}$. $\Sf$ and $S_{Kt}$ are the only distributions that the authors are aware of that satisfy these two properties.

Let $\mathcal{C}_t$ denote the set of strings $x$ that are incomputable in time $t$ (that is, there is no program $p$ such that $p \rightarrow x$ in $t$ or fewer timesteps) such that for any $y \sqsubset x$, the prefix $y$ is computable in time $t$. By definition, all strings that are incomputable in time $t$ have as a prefix an element of $\mathcal{C}_t$, and $\mathcal{C}_t$ is a prefix-free set\footnote{That is, a set such that no element is a prefix of another element.} (by construction). Furthermore, the probability measure of all strings incomputable in time $t$ is simply the sum of the probabilities of all elements of $\mathcal{C}_t$.

\begin{prop} \label{prop:low_prior_expensive_string}
  \begin{equation*}
    \sum_{x \in \mathcal{C}_t} S_{Kt}(x) \leq \frac{1}{t}
  \end{equation*}
\end{prop}

\begin{proof}
\begin{align*}
  \sum_{x \in \mathcal{C}_t} S_{Kt}(x) &= \sum_{x \in \mathcal{C}_t} \sum_{p \rightarrow x} \frac{2^{-|p|}}{t(p,x)} 
                                         \leq \frac{1}{t} \sum_{x \in \mathcal{C}_t}\sum_{p \rightarrow x} 2^{-|p|} \leq \frac{1}{t}
\end{align*}
by the Kraft inequality, since the fact that $\mathcal{C}_t$ is a prefix-free set guarantees that the set of programs that compute elements of $\mathcal{C}_t$ is also prefix-free, due to our use of monotone machines.
\end{proof}

\begin{prop}\label{prop:s_kt_speed_prior_2}
  Let $x_{1:\infty} \in \B^\infty$ be such that there exists a program $p^x \in \B^*$ which outputs $x_{1:n}$ in $f(n)$ steps for all $n \in \N$. Let $g(n)$ grow faster than $f(n)$, i.e.\ $\lim_{n \rightarrow \infty} f(n)/g(n) = 0$. Then,
\begin{equation*}
  \lim_{n \rightarrow \infty} \frac{\sum_{p \xrightarrow[\geq g(n)]{}  x_{1:n}}  2^{-|p|} / t(p, x_{1:n})}{\sum_{p \xrightarrow[\leq f(n)]{}  x_{1:n}} 2^{-|p|}/t(p, x_{1:n})} = 0
\end{equation*}
where $p \xrightarrow[\leq t]{} x$ iff program $p$ computes string $x$ in no more than $t$ steps.
\end{prop}

An informal statement of this proposition is that contributions to $S_{Kt}(x_{1:n})$ by programs that take time longer than $g(n)$ steps to run are dwarfed by those by programs that take less than $f(n)$ steps to run. Therefore, asymptotically, only the fastest programs contribute to $S_{Kt}$.

\begin{proof}
  \begin{align}
\nonumber &\quad \lim_{n \rightarrow \infty} \frac{\sum_{p \xrightarrow[\geq g(n)]{}  x_{1:n}}  2^{-|p|}/t(p, x_{1:n})}{\sum_{p \xrightarrow[\leq f(n)]{}  x_{1:n}} 2^{-|p|}/t(p, x_{1:n})} \\
 &\leq \lim_{n \rightarrow \infty} \frac{\sum_{p \xrightarrow[\geq g(n)]{} x_{1:n} } 2^{-|p|}/g(n)} {2^{-|p^x|}/f(n)} \label{eq:proofs4}\\
 &\leq \lim_{n \rightarrow \infty} \frac{f(n)}{g(n)} \frac{\sum_{p \rightarrow x_{1:n}} 2^{-|p|}}{2^{-|p^x|} } \label{eq:proofs11}\\
 &\leq \lim_{n \rightarrow \infty} \frac{f(n)}{g(n)} \frac{1}{2^{-|p^x|}} \label{eq:proofs27}\\
\nonumber &= 0
\end{align}
Equation \eqref{eq:proofs4} comes from increasing $1/t(p, x_{1:n})$ to $1/g(n)$ in the numerator, and decreasing the denominator by throwing out all terms of the sum except that of $p^x$, which takes $f(n)$ time to compute $x_{1:n}$. Equation \eqref{eq:proofs11} takes $f(n)/g(n)$ out of the fraction, and increases the numerator by adding contributions from all programs that compute $x_{1:n}$. Equation \eqref{eq:proofs27} uses the Kraft inequality to bound $\sum_{p \rightarrow x_{1:n}} 2^{-|p|}$ from above by 1. Finally, we use the fact that $\lim_{n \rightarrow \infty} f(n)/g(n) = 0$. 
\end{proof}


\section{Loss bounds}
\label{sec:loss-bounds}

In this section, we prove a performance bound on $S_{Kt}$-based sequence prediction, when predicting a sequence drawn from a measure that is estimable in polynomial time. We also prove a similar bound on $\Sf$-based sequence prediction when predicting deterministic sequences computable in polynomial time.

For the purpose of this section, we write $S_{Kt}$ somewhat more explicitly as
\begin{equation*}
  S_{Kt}(x) = \sum_{p \xrightarrow{U} x} \frac{ 2^{-|p|}}{t(U,p,x)}
\end{equation*}
and give some auxiliary definitions. Let $\langle \cdot \rangle_{\B^*}$ be a prefix-free coding of the strings of finite length and $\langle  \cdot \rangle_{\N}$ be a prefix-free coding of the integers, where both of these prefix-free codings are computable and decodable in polynomial time.

\begin{mydef}
  A function $f : \B^* \rightarrow \R$ is \emph{finitely computable} if there is some Turing machine $T_f$ that when given input $\langle x \rangle_{\B^*}$ prints $\langle m \rangle_{\N} \langle n \rangle_{\N}$ and then halts, where $f(x) = m/n$. The function $f$ is \emph{finitely computable in polynomial time} if it takes $T_f$ at most $p(|x|)$ timesteps to halt on input $x$, where $p$ is a polynomial.
\end{mydef}

\begin{mydef}
  Let $f,g: \B^* \rightarrow \R$. $g$ is \emph{estimable in polynomial time by $f$} if $f$ is finitely computable in polynomial time and $f(x) \timeseq g(x)$. The function $g$ is \emph{estimable in polynomial time} if it is estimable in polynomial time by some function $f$.
\end{mydef}

First, note that this definition is reasonably weak, since we only require $f(x) \timeseq g(x)$, rather than $f(x) = g(x)$. Also note that if $f$ is finitely computable in polynomial time, it is estimable in polynomial time by itself. For a measure $\mu$, estimability in polynomial time captures our intuitive notion of efficient computability: we only need to know $\mu$ up to a constant factor for prediction, and we can find this out in polynomial time.

We consider a prediction setup where a predictor outputs a prediction, and then receives some loss depending on the predicted next bit and the correct next bit. More formally, we have some loss function $\ell(x_t, y_t) \in [0,1]$ defined for all $x_t, y_t \in \B$ and all $t \in \N$, representing the loss incurred for a prediction of $y_t$ when the actual next bit is $x_t$, which the predictor observes after prediction. One example of such a loss function is the 0-1 loss, which assigns 0 to a correct prediction and 1 to an incorrect prediction, although there are many others.

We define the $\Lambda_\rho$ predictor to be the predictor which minimises $\rho$-expected loss, outputting $y^{\Lambda_\rho}_t := \argmin_{y_t} \, \sum_{x_t} \rho(x_t | x_{1:t-1}) \ell(x_t, y_t)$ at time $t$. If the true distribution is $\mu$, we judge a predictor $\Lambda$ by its total $\mu$-expected loss in the first $n$ steps:
\begin{equation*}
  L^\Lambda_{n\mu} := \E_\mu \left[ \sum_{t = 1}^n \ell(x_t, y_t^\Lambda) \right]
\end{equation*}
In particular, if we are using 0-1 loss, $L^\Lambda_{n\mu}$ is the expected number of errors made by $\Lambda$ up to time $n$ in the environment $\mu$.

\begin{thm}[\emph{Bound on $S_{Kt}$ prediction loss}]\label{thm:error-bound-skt}
  If $\mu$ is a measure that is estimable in polynomial time by some semimeasure $\nu$, and $x$ is a sequence sampled from $\mu$, then the expected loss incurred by the $\Lambda_{S_{Kt}}$ predictor is bounded by
  \begin{equation*}
    L^{\Lambda_{S_{Kt}}}_{n\mu} - L^{\Lambda_\mu}_{n\mu} \leq 2 D_n + 2 \sqrt{L^{\Lambda_\mu}_{n\mu}D_n}
  \end{equation*}
  where $D_n = O(\log n)$.\footnote{A similar bound that can be proved the same way is $\sqrt{L^{\Lambda_{S_{Kt}}}_{n\mu}} - \sqrt{L^{\Lambda_\mu}_{n\mu}} \leq \sqrt{2D_n}$ for the same $D_n$~\citep[Eq.\ 8, 5]{hutter2007universalbayes}.}
\end{thm}
Since $L^{\Lambda_\mu}_{n\mu} \leq n$, this means that $\Lambda_{S_{Kt}}$ only incurs at most $O(\sqrt{n \log n})$ extra loss in expectation, although this bound will be much tighter in more structured environments where $\Lambda_\mu$ makes few errors, such as deterministic environments.

In order to prove this theorem, we use the following lemma:

\begin{lem} \label{lem:alg-coding}
Let $\nu$ be a semimeasure that is finitely computable in polynomial time. There exists a Turing machine $T_\nu$ such that for all $x \in \B^*$
\begin{equation}
  \nu(x) = \sum_{p \xrightarrow{T_\nu} x} 2^{-|p|} \label{eq:31}
\end{equation}
and
\begin{equation}
2^{-Km_{T_{\nu}}(x)} \geq \nu(x)/4 \label{eq:32}
\end{equation}
where $Km_{T_\nu}(x)$ is the length of the shortest program for $x$ on $T_\nu$.\footnote{Note that this lemma would be false if we were to let $\nu$ be an arbitrary lower-semicomputable semimeasure, since if $\nu = M$, this would imply that $2^{-Km(x)} \timeseq M(x)$, which was disproved by \cite{gacs1983relation}.}
\end{lem}

Note that a proof already exists that there is some machine $T_\nu$ such that \eqref{eq:31} holds~\citep[Thm.\ 4.5.2]{livitanyi}, but it does not prove \eqref{eq:32}, and we wish to understand the operation of $T_\mu$ in order to prove Theorem~\ref{thm:error-bound-skt}.

\begin{proof}[Proof of Lemma~\ref{lem:alg-coding}]
  The machine $T_\nu$ is essentially a decoder of an algorithmic coding scheme with respect to $\nu$. It uses the natural correspondence between $\B^\infty$ and $[0,1]$, associating a binary string $x_1 x_2 x_3 \dotsb$ with the real number $0.x_1 x_2 x_3 \dotsb$. It determines the location of the input sequence on this line, and then assigns a certain interval for each output string, such that the width of the interval for output string $x$ is equal to $\nu(x)$. Then, if input string $p$ lies inside the interval for the output string $x$, it outputs $x$.

  $T_\nu$ first calculates $\nu(0)$ and $\nu(1)$, and sets $[0, \nu(0))$ as the output interval for 0 and $[\nu(0), \nu(0) + \nu(1))$ as the output interval for 1. It then reads the input, bit by bit. After reading input $p_{1:n}$, it constructs the input interval $[0.p_1 p_2 \dotsb p_n, 0.p_1 p_2 \dotsb p_n 1 1 1 1 1 1 \dotsb)$, which represents the inerval that $0.p_1 p_2 \dotsb p_n p_{n + 1} \dotsb$ could lie in. It then checks if this input interval is contained in one of the output intervals. If it is, then it prints output appropriate for the interval, and if not, then it reads one more bit and repeats the process.

  Suppose the first output bit is a 1. Then, $T_\nu$ calculates $\nu(10)$ and $\nu(11)$, and forms the new output intervals: $[\nu(0), \nu(0) + \nu(10))$ for outputting 0, and $[\nu(0) + \nu(10), \nu(0) + \nu(10) + \nu(11))$ for outputting 1. It then reads more input bits until the input interval lies within one of these new output intervals, and then outputs the appropriate bit. The computation proceeds in this fashion.

  Equation \eqref{eq:31} is satisfied, because $\sum_{p \xrightarrow{T_\nu} x} 2^{-|p|}$ is just the total length of all possible input intervals that fit inside the output interval for $x$, which by construction is $\nu(x)$.

  To show that \eqref{eq:32} is satisfied, note that $2^{-Km_{T_\nu}(x)}$ is the length of the largest input interval for $x$. Now, input intervals are binary intervals (that is, their start points and end points have a finite binary expansion), and for every interval $I$, there is some binary interval contained in $I$ with length $\geq 1/4$ that of $I$. Therefore, the output interval for $x$ contains some input interval with length at least $1/4$ that of the length of the output interval. Since the length of the output interval for $x$ is just $\nu(x)$, we can conclude that $2^{-Km_{T_\nu}(x)} \geq \nu(x)/4$.
\end{proof}

\begin{proof}[Proof of Theorem~\ref{thm:error-bound-skt}]
  Using Lemma~\ref{lem:alg-coding}, we show a bound on $S_{Kt}$ that bounds its KL divergence with $\mu$. We then apply the unit loss bound~\citep[Thm.\ 3.48]{marcus2005universal} (originally shown for the Solomonoff prior, but valid for any prior) to show the desired result.

  First, we reason about the running time of the shortest program that prints $x$ on the machine $T_\nu$ (defined in Lemma~\ref{lem:alg-coding}\hspace{0pt}). Since we would only calculate $\nu(y0)$ and $\nu(y1)$ for $y \sqsubseteq x$, this amounts to $2|x|$ calculations. Each calculation need only take polynomial time in the length of its argument, because $T_\nu$ could just simulate the machine that takes input $x$ and returns the numerator and denominator of $x$, prefix-free coded, and it only takes polynomial time to undo this prefix-free coding. Therefore, the calculations take at most $2|x|f(|x|) =: g(|x|) $, where $f$ is a polynomial. We also, however, need to read all the bits of the input, construct the input intervals, and compare them to the output intervals. This takes time linear in the number of bits read, and for the shortest program that prints $x$, this number of bits is (by definition) $Km_{T_\nu}(x)$. Since $2^{-Km_{T_\nu(x)}} \timeseq \nu(x)$, $Km_{T_\nu}(x) \leq -\log(\nu(x)) + O(1)$, and since $\nu(x) \timeseq \mu(x)$, $-\log(\nu(x)) \leq -\log(\mu(x)) + O(1)$. Therefore, the total time taken is bounded above by $g(|x|) - O(1)\log(\mu(x))$, where we absorb the additive constants into $g(|x|)$.

  This out of the way, we can calculate
  \begin{align}
\nonumber    S_{Kt}(x) &= \sum_{p \xrightarrow{U} x} \frac{2^{-|p|}}{t(U,p,x)} \\
\nonumber                       &= \sum_{\text{Turing machines }T} 2^{-|I(T)|}  \sum_{q \xrightarrow{T} x} \frac{2^{-|q|}}{t(T,q,x )^{O(1)}}  \\
    \nonumber              &\timesgeq \sum_{p \xrightarrow{T_\nu} x} \frac{ 2^{-|p|} }{t(T_\nu,p,x)^{O(1)}}\\
\nonumber              &\geq  \frac{2^{-Km_{T_\nu}(x)}}{(g(|x|) - O(1) \log (\mu(x)))^{O(1)}}  \\
    &\timesgeq \frac{\mu(x)}{(g(|x|) - O(1) \log (\mu(x)))^{O(1)}} \label{eq:33}
  \end{align}
  Now, the unit loss bound tells us that
\begin{equation}
    L_{n\mu}^{\Lambda_{S_{Kt}}} - L_{n\mu}^{\Lambda_\mu} \leq 2D_n(\mu || S_{Kt}) + 2 \sqrt{L_{n\mu}^{\Lambda_\mu} D_n(\mu || S_{Kt})} \label{eq:18}
\end{equation}
where $  D_n(\mu || S_{Kt}) := \E_\mu \left[ \ln\left( \mu (x_{1:n})/S_{Kt}(x_{1:n})\right) \right]$ is the relative entropy. We can calculate $D_n(\mu || S_{Kt})$ using equation \eqref{eq:33}:
\begin{align}
  \nonumber D_n(\mu || S_{Kt}) &= \E_\mu \left[ \ln \frac{\mu(x_{1:n})}{S_{Kt}(x_{1:n})} \right] \\
\nonumber                               &\timesleq \E_\mu \left[ \ln \left((g(n) - O(1) \log (\mu(x_{1:n}) ) )^{O(1)}\right) \right] \\
  \nonumber  &\timesleq \E_\mu \left[ \ln (g(n) - O(1) \log (\mu(x_{1:n}))) \right] \\
 &\leq \ln \E_\mu \left[ g(n) - O(1) \log (\mu(x_{1:n})) \right] \label{eq:35} \\
  \nonumber &= \ln \left( g(n) + O(1) H_\mu(x_{1:n}) \right) \\
  \intertext{where $H_\mu(x_{1:n})$ denotes the binary entropy of the random variable $x_{1:n}$ with respect to $\mu$}
   \nonumber                            &\leq \ln \left( g(n) + O(n) \right) \\
  &\timeseq \log n  \label{eq:34}
\end{align}
where \eqref{eq:35} comes from Jensen's inequality. Equations \eqref{eq:18} and \eqref{eq:34} together prove the theorem.
\end{proof}

We therefore have a loss bound on the $S_{Kt}$-based sequence predictor in environments that are estimable in polynomial time by a semimeasure. Furthermore:
\begin{cor}
\begin{equation*}
  L_{n\mu}^{\Lambda_{S_{Kt}}} \leq 2 D_n(\mu || S_{Kt}) \timeseq \log n
\end{equation*}
for deterministic measures\footnote{That is, measures that give probability 1 to prefixes of one particular infinite sequence.} $\mu$ computable in polynomial time, if correct predictions incur no loss.
\end{cor}

We should note that this method fails to prove similar bounds for $S_\text{Fast}$, since we instead get
\begin{align}
 S_{\text{Fast}}(x) &\timeseq \sum_{p \xrightarrow{U} x} \frac{ 2^{-2|p|}}{t(U,p,x)} \timesgeq \frac{\mu(x)^2}{(|x|^{O(1)} - \log \mu(x))^{O(1)}} \label{eq:25}
\end{align}

which gives us
\begin{align*}
  D_n (\mu || \Sf) &= \E_\mu \left[ \ln \frac{\mu(x_{1:n})}{S_{\text{Fast}}(x_{1:n})}  \right] \\
  & \leq O(\log n) + H_\mu(x_{1:n}) 
\end{align*}
Since $H_\mu(x_{1:n})$ can grow linearly in $n$ (for example, take $\mu$ to be $\lambda(x) = 2^{-|x|}$, the uniform measure), this can only prove a trivial linear loss bound without restrictions on the measure $\mu$. It is also worth explicitly noting that the constants hidden in the $O(\cdot)$ notation depend on the environment $\mu$, as will be the case for the rest of this paper.

One important application of Theorem~\ref{thm:error-bound-skt} is to the 0-1 loss function. Then, it states that a predictor that outputs the most likely successor bit according to $S_{Kt}$ only makes logarithmically many errors in a deterministic environment computable in polynomial time. In other words, $S_{Kt}$ quickly learns the sequence it is predicting, making very few errors.

Next, we show that $\Sf$ makes only logarithmically many errors on a sequence deteriministically computed in polynomial time. This follows from a rather simple argument.

\begin{thm}[\emph{Bound on $\Sf$ prediction loss}]\label{thm:sfast-predicts-det-poly}
  Let $\mu$ be a deterministic environment and $x_{1:\infty}$ be the sequence whose prefixes $\mu$ assigns probability 1 to. If $x_{1:\infty}$ is computable in polynomial time by a program $p^x$, then $\Sf$ only incurrs logarithmic loss, if correct predictions incur no loss.
\end{thm}

\begin{proof}
  Using the unit loss bound,
  \begin{align*}
  L_{n\mu}^{\Lambda_{\Sf}} &= L_{n\mu}^{\Lambda_{\Sf}} - L^{\Lambda_\mu}_{n \mu}\\
                           &\leq 2 D_n(\mu || \Sf) \\
                           &= -2 \ln \Sf (x_{1:n}) \\
                           &\timesleq 2 |p^x| + \log t(p^x, x_{1:n}) \\
                           &\timeseq \log n
  \end{align*}

  \vspace{-4.9ex}
\end{proof}


\section{Time complexity}
\label{sec:time-complexity}

Although it has been proved that $\Sf$ is computable~\citep{schmidhuber2002speed}, no bounds are given for its computational complexity. Given that the major advantage of $\Sf$-based prediction over $M$-based prediction is its computability, it is of interest to determine the time required to compute $\Sf$, and whether such a computation is feasible or not. The same questions apply to $S_{Kt}$, to a greater extent because we have not even yet shown that $S_{Kt}$ is computable.

In this section, we show that an arbitrarily good approximation to $\Sf(x)$ is computable in time exponential in $|x|$, and an arbitrarily good approximation to $S_{Kt}(x)$ is computable in time doubly-exponential in $|x|$. We do this by explicitly constructing algorithms that perform \textsc{phase}s of \textsc{fast} until enough contributions to $\Sf$ or $S_{Kt}$ are found to constitute a sufficient proportion of the total.

We also show that no such approximation of $S_{Kt}$ or $\Sf$ can be computed in polynomial time. We do this by contradiction: showing that if it were possible to do so, we would be able to construct an `adversarial' sequence that was computable in polynomial time, yet could not be predicted by our approximation; a contradiction.

Finally, we investigate the time taken to compute $S_{Kt}$ and $\Sf$ along a polynomial-time computable sequence $x_{1:\infty}$. If we wanted to predict the most likely continuation of $x_{1:n}$ according to $S \in \{S_{Kt}, \Sf\}$, we would have to compute an approximation to $S(x_{1:n} 0)$ and $S(x_{1:n} 1)$, to see which one was greater. We show that it is possible to compute these approximations in polynomial time for $\Sf$ and in exponential time for $S_{Kt}$: an exponential improvement over the worst-case bounds in both cases.

\subsection{Upper bounds}
\label{sec:s_textfast}

\begin{thm}[\emph{$\Sf$ computable in exponential time}]\label{thm:sf_comp_exp}
  For any $\vareps > 0$, there exists an approximation $\Sf^\vareps$ of $\Sf$ such that $| \Sf^\vareps / \Sf - 1| \leq \vareps$ and $\Sf^\vareps(x)$ is computable in time exponential in $|x|$.
\end{thm}

\begin{proof}
First, we note that in \textsc{phase} $i$ of \textsc{fast}, we try out $2^1 + \dotsb + 2^i = 2^{i + 1}$ program prefixes $p$, and each prefix $p$ gets $2^{i - |p|}$ steps. Therefore, the total number of steps in \textsc{phase} $i$ is $2^1 \times 2^{i - 1} + 2^2 \times 2^{i - 2} + \dotsb + 2^i \times 2^{i - i} = i 2^i$, and the total number of steps in the first $k$ \textsc{phase}s is
\begin{equation}
\text{\# steps} =  \sum_{i = 1}^k i 2^i = 2^{k + 1}(k - 1) + 2 \label{eq:22}
\end{equation}

Now, suppose we want to compute a sufficient approximation $S_{\text{Fast}}^\varepsilon(x)$. If we compute $k$ phases of \textsc{fast} and then add up all the contributions to $S_{\text{Fast}}(x)$ found in those phases, the remaining contributions must add up to $\leq \sum_{i = k + 1}^\infty 2^{-i} = 2^{-k}$. In order for the contributions we have added up to contribute $ \geq 1 - \varepsilon$ of the total, it suffices to use $k$ such that
\begin{align}
k &= \left\lfloor -\log (\vareps S_\text{Fast} (x)) + 1 \right\rfloor \label{eq:28}
\end{align}

Now, since the uniform measure $\lambda(x) = 2^{-|x|}$ is finitely computable in polynomial time, it is estimable in polynomial time by itself, so we can substitute $\lambda$ into equation \eqref{eq:25} to obtain
\begin{align}
S_\text{Fast}(x) &\timesgeq \frac{2^{-2|x|}}{(|x|^{O(1)} + \log(2^{|x|}))^{O(1)}} = \frac{1}{|x|^{O(1)} 2^{2|x|}} \label{eq:24}
\end{align}
Substituting equation \eqref{eq:24} into equation \eqref{eq:28}, we get
\begin{align}
\nonumber  k &\leq \log \left( O(2^{2|x|} |x|^{O(1)})/ \vareps \right) + 1 \\
  & = - \log \vareps + 2|x| + O(\log |x|)  \label{eq:30}
\end{align}
So, substituting equation \eqref{eq:30} into equation \eqref{eq:22},
\begin{align*}
  \text{\# steps} &\leq 2^{- \log \vareps + 2|x| + O(\log |x|) + 1}\\
  &\quad {} \times (- \log \vareps + 2|x| + O(\log |x|) - 1) + 2 \\
  &= \frac{1}{\vareps} 2^{2|x|} |x|^{O(1)} (- \log \vareps + 2|x| + O(\log |x|)) \\
  &\leq 2^{O(|x|)}
\end{align*}
Therefore, $S_\text{Fast}^\vareps$ is computable in exponential time.
\end{proof}

\begin{thm}[\emph{$S_{Kt}$ computable in doubly-exponential time}]\label{thm:skt_comp_double_exp}
  For any $\vareps > 0$, there exists an approximation $S_{Kt}^\vareps$ of $S_{Kt}$ such that $|S_{Kt}^\vareps / S_{Kt} - 1| \leq \vareps$ and $S_{Kt}^\vareps$ is computable in time doubly-exponential in $|x|$.
\end{thm}

\begin{proof}
We again use the general strategy of computing $k$ \textsc{phase}s of \textsc{fast}, and adding up all the contributions to $S_{Kt}(x)$ we find. Once we have done this, the other contributions come from computations with $Kt$-cost $> k$. Therefore, the programs making these contributions either have a program of length $> k$, or take time $> 2^k$ (or both).

First, we bound the contribution to $S_{Kt}(x)$ by computations of time $> 2^k$:
\begin{equation*}
 \sum_{p \xrightarrow[>2^k]{} x} \frac{2^{-|p| }}{t(p,x)} <  \frac{1}{2^k} \sum_{p \rightarrow x} 2^{-|p|} \leq \frac{1}{2^k}
\end{equation*}

Next, we bound the contribution by computations with programs of length $|p| > k$. We note that since we are dealing with monotone machines, the worst case is that all programs have length $k + 1$, and the time taken is only $k + 1$ (since, by the definition of monotone machines, we need at least enough time to read the input). Then, the contribution from these programs is $2^{k + 1} \times (1/(k + 1)) \times 2^{-k -1} = 1/(k + 1)$, meaning that the total remaining contribution after $k$ \textsc{phase}s is no more than $2^{-k} + 1/(k + 1) \leq 2/(k + 1)$.

So, in order for our contributions to add up to $\geq 1 - \vareps$ of the total, it suffices to use $k$ such that
\begin{equation}
k = \left\lfloor 2 (\vareps S_{Kt}(x))^{-1} \right\rfloor \label{eq:proofs3}
\end{equation}
Now, again since $\lambda$ is finitely computable in polynomial time, we substitute it into equation (5) to obtain
\begin{equation}
  S_{Kt}(x) \timesgeq \frac{1}{|x|^{O(1)}2^{|x|}} \label{eq:proofs23}
\end{equation}
Substituting equation \eqref{eq:proofs23} into equation \eqref{eq:proofs3}, we get
\begin{equation}
  \label{eq:proofs29}
  k \leq O(|x|^{O(1)}2^{|x|})/\vareps
\end{equation}
So, substituting equation \eqref{eq:proofs29} into equation (\ref{eq:22}), we finally obtain
\begin{align*}
  \text{\# steps} &\leq 2^{O(|x|^{O(1)} 2^{|x|}) / \vareps} \left(\frac{O(|x|^{O(1)}2^{|x|})}{\vareps} \right) + 2 \\
  &\leq 2^{2^{O(|x|)}}
\end{align*}
Therefore, $S_{Kt}^\vareps$ is computable in doubly-exponential time.
\end{proof}

\subsection{Lower bounds}
\label{sec:s_kt-not-computable}

\begin{thm}[\emph{$S_{Kt}$ not computable in polynomial time}] \label{thm:s_kt-not-poly-comp}
  For no $\vareps > 0$ does there exist an approximation $S^\vareps_{Kt}$ of $S_{Kt}$ such that $|S_{Kt}^\vareps / S_{Kt} - 1| \leq \vareps$ and $S_{Kt}^\vareps$ is computable in time polynomial in $|x|$.
\end{thm}

The proof of this theorem relies on the following lemma:
\begin{lem}\label{lem:skteps_just_as_good}
  If $S^\vareps_{Kt}$ is an approximation of $S_{Kt}$ as given in Theorem~\ref{thm:s_kt-not-poly-comp}, then the bound of Theorem~\ref{thm:error-bound-skt} applies to $S^\vareps_{Kt}$. That is,
  \begin{equation*}
    L_{n\mu}^{\Lambda_{S_{Kt}^\vareps}} - L_{n\mu}^{\Lambda_\mu} \leq 2 D_n + 2 \sqrt{L_{n\mu}^{\Lambda_\mu} D_n}
  \end{equation*}
where $D_n = O(\log n)$.
\end{lem}

\begin{proof}[Proof of Lemma~\ref{lem:skteps_just_as_good}]
  From the definition of $S^\vareps_{Kt}$, it is clear that $S^\vareps_{Kt} \geq (1 - \vareps) S_{Kt}$. Then,
  \begin{align*}
 D_n(\mu || S^\vareps_{Kt}) &:= \E_\mu \left[ \ln \frac{\mu(x_{1:n})}{S^\vareps_{Kt}(x_{1:n})} \right] \\
                                        &\leq \E_\mu \left[ \ln \frac{\mu(x_{1:n})}{S_{Kt}(x_{1:n})}\right] - \ln (1 - \vareps) \\
    &\timeseq \log n
  \end{align*}
for $\mu$ estimable in polynomial time by a semimeasure, where we use Theorem~\ref{thm:error-bound-skt} for the final `equality'. Therefore, the bound of Theorem~\ref{thm:error-bound-skt} applies. 
\end{proof}

\begin{proof}[Proof of Theorem~\ref{thm:s_kt-not-poly-comp}]
  Suppose by way of contradiction that $S_{Kt}^\vareps$ were computable in polynomial time. Then, the sequence $x_{1:\infty}$ would also be computable in polynomial time, where
\begin{align*}
  x_n =
  \begin{dcases}
    1 & \text{if } S_{Kt}^\vareps(0 | x_{1:n-1}) \geq S_{Kt}^\vareps(1 | x_{1:n-1}) \\
    0 & \text{if } S_{Kt}^\vareps(0 | x_{1:n-1}) < S_{Kt}^\vareps(1 | x_{1:n-1})
  \end{dcases}
\end{align*}
$x_{1:\infty}$ is therefore an adversarial sequence against $S_{Kt}^\vareps$: it predicts whichever symbol $S_{Kt}^\vareps$ thinks less likely, and breaks ties with 1.

Now, consider an $S_{Kt}^\vareps$-based predictor $\Lambda_{S^\vareps_{Kt}}$ that minimises 0-1 loss---that is, one that predicts the more likely continuation according to $S_{Kt}^\vareps$. Further, suppose this predictor breaks ties with 0. Since the loss bound of Theorem~\ref{thm:error-bound-skt} applies independently of tie-breaking method, Lemma~\ref{lem:skteps_just_as_good} tells us that $\Lambda_{S^\vareps_{Kt}}$ must make only logarithmically many errors when predicting $x_{1:\infty}$. However, by design, $\Lambda_{S^\vareps_{Kt}}$ errs every time when predicting $x_{1:\infty}$. This is a contradiction, showing that $S_{Kt}^\vareps$ cannot be computable in polynomial time.
\end{proof}

Next, we provide a proof of the analogous theorem for Schmidhuber's speed prior $\Sf$, using a lemma about the rate at which $\Sf$ learns polynomial-time computable deterministic sequences.

\begin{thm}[\emph{$\Sf$ not computable in polynomial time}]\label{thm:sf-not-poly-comp}
  For no $\vareps > 0$ does there exist an approximation $\Sf^\vareps$ of $\Sf$ such that $|\Sf^\vareps/\Sf - 1| \leq \vareps$ and $\Sf^\vareps(x)$ is computable in time polynomial in $|x|$.
\end{thm}

\begin{lem}\label{lem:sf-close-to-poly-comp-seq}
  For a sequence $x_{1:\infty}$ computed in polynomial time by some program $p^x$,
  \begin{equation*}
    \sum_{t=1}^n |1 - \Sf(x_t \mid x_{<t})| \timesleq \log n
  \end{equation*}
\end{lem}

\begin{proof}[Proof of Lemma~\ref{lem:sf-close-to-poly-comp-seq}]
  We calculate
  \begin{align*}
    &\quad \sum_{t=1}^n |1 - \Sf(x_t \mid x_{<t})|  \\
    &\leq - \sum_{t = 1}^n \ln \Sf(x_t \mid x_{<t}) \\
    & = - \ln \prod_{t=1}^n \Sf(x_t \mid x_{<t}) \\
    &= - \ln \Sf(x_{1:n}) \\
    &\timesleq 2 |p^x| + \log t(p^x, x_{1:n}) \\
    &\timesleq \log n
  \end{align*}

  \vspace{-4.9ex}
\end{proof}

\begin{proof}[Proof of Theorem~\ref{thm:sf-not-poly-comp}]
  Let $\Sf^\vareps$ be computable in polynomial time, and construct the adversarial sequence $x_{1:\infty}$ against $\Sf^\vareps$ in the same manner as in the proof of Theorem~\ref{thm:s_kt-not-poly-comp}. Then, $x_{1:\infty}$ would be a deterministic sequence computable in polynomial time, and so by Lemma~\ref{lem:sf-close-to-poly-comp-seq},
  \begin{align*}
    \log n &\timesgeq \sum_{t=1}^n |1 - \Sf(x_t \mid x_{<t})| \\
 &\geq \sum_{t=1}^n |1 - \Sf^\vareps (x_t \mid x_{<t}) | - \vareps n \\
    &\geq \left( \frac{1}{2} - \vareps \right) n
  \end{align*}
  a contradiction. Therefore, $\Sf^\vareps$ cannot be computable in polynomial time.
\end{proof}

Note the similarity between the speed priors and $M$: all succeed at predicting sequences in a certain computability class, and therefore none are in that class.

\subsection{Computability along polynomial time computable sequences}
\label{sec:comp-along-polyn}

\begin{thm}[\emph{$\Sf$ computable in polynomial time on polynomial time computable sequence}]\label{thm:sfast-poly-poly}
  If $x_{1:\infty}$ is computable in polynomial time, then $\Sf^\vareps(x_{1:n}0)$ and $\Sf^\vareps(x_{1:n}1)$ are also computable in polynomial time.
\end{thm}

\begin{proof}
    Suppose some program $p^x$ prints $x_{1:\infty}$ in time $f(n)$, where $f$ is a polynomial. Then,
  \begin{equation*}
    \Sf(x_{1:n}) \geq \frac{2^{-2|p^x|}}{f(n)} 
  \end{equation*}
  Substituting this into equation \eqref{eq:28}, we learn that to compute $\Sf^\vareps(x_{1:n})$, we need to compute \textsc{fast} for $k$ \textsc{phase}s where
  \begin{equation*}
    k \leq \left\lfloor \log (2^{2|p^x|}f(n)/ \vareps) \right\rfloor
  \end{equation*}
  Substituting this into equation \eqref{eq:22} gives
  \begin{align*}
    \#\text{ steps} &\leq 2^{\log(2^{2|p^x|}f(n)/\vareps)}(\log (2^{2|p^x|}f(n)/\vareps) - 1) + 2 \\
    &= O(f(n) \log f(n)) = O(f(n) \log n)
  \end{align*}
  Therefore, we only require a polynomial number of steps of the \textsc{fast} algorithm to compute $\Sf^\vareps(x_{1:n})$. To prove that it only takes a polynomial number of steps to compute $\Sf^\vareps(x_{1:n}b)$ for any $b \in \B$ requires some more careful analysis.

  Let $\langle n \rangle$ be a prefix-free coding of the natural numbers in $2\log n$ bits. Then, if $b \in \B$, then there is some program prefix $p^b$ such that $p^b \langle n \rangle q$ runs program $q$ until it prints $n$ symbols on the output tape, after which it stops running $q$, prints $b$, and then halts. In addition to running $q$ (possibly slowed down by a constant factor), it must run some sort of timer to count down to $n$. This involves reading and writing the integers 1 to $n$, which takes $O(n \log n)$ time. Therefore, $p^b \langle  n \rangle p^x$ prints $x_{1:n}b$ in time $O(f(n)) + O(n \log n)$, so
  \begin{align*} 
    \Sf(x_{1:n}b) &\geq \frac{2^{-2|p^b \langle  n \rangle p^x|}}{O(f(n)) + O(n \log n)}  \\
                  &= \frac{1}{O(f(n)) + O(n \log n)} \frac{1}{n^4 2^{2|p^b| + 2|p^x|}}\\
    &= \frac{1}{g(n)}
  \end{align*}
  for some polynomial $g$ of degree 4 greater than the degree of $f$. Using equations \eqref{eq:28} and \eqref{eq:22} therefore gives that we only need $O(g(n) \log g(n)) = O(g(n) \log n)$ timesteps to compute $\Sf^\vareps(x_{1:n}b)$. Therefore, both $\Sf^\vareps(x_{1:n} 0)$ and $\Sf^\vareps(x_{1:n}1)$ are computable in polynomial time.
\end{proof}

Note that the above proof easily generalises to the case where $f$ is not a polynomial.

\begin{thm}[\emph{$S_{Kt}$ computable in exponential time on polynomial time computable sequence}]\label{thm:skt-poly-exp}
  If $x_{1:\infty}$ is computable in polynomial time, then $S_{Kt}^\vareps(x_{1:n}0)$ and $S_{Kt}^\vareps(x_{1:n}1)$ are computable in time $2^{n^{O(1)}}$.
\end{thm}

\begin{proof}
    The proof is almost identical to the proof of Theorem \ref{thm:sfast-poly-poly}: supposing that $p^x$ prints $x_{1:n}$ in time $f(n)$ for some polynomial $f$, we have
  \begin{equation*}
    S_{Kt}(x_{1:n}) \geq \frac{2^{-|p^x|}}{f(n)}
  \end{equation*}
  The difference is that we substitute this into equation \eqref{eq:proofs3}, getting
  \begin{equation*}
    k \leq \left\lfloor 2^{|p^x| + 1}f(n)/\vareps\right\rfloor
  \end{equation*}
  and substitution into equation \eqref{eq:22} now gives
  \begin{align*}
    \# \text{ steps} &\leq 2^{2^{|p^x| + 1} f(n)/\vareps}\left( 2^{|p^x| + 1} f(n)/ \vareps - 1 \right) + 2\\
    &= 2^{O(f(n))}
  \end{align*}
  The other difference is that when we bound $S_{Kt}(x_{1:n}b) \geq 1/g(n)$, the degree of $g$ is only 2 greater than that of the degree of $f$. Therefore, we can compute $S_{Kt}^\vareps(x_{1:n}0)$ and $S_{Kt}^\vareps(x_{1:n}1)$ in time $2^{n^{O(1)}}$.
\end{proof}

Note that Theorem \ref{thm:sfast-poly-poly} does not contradict Theorem \ref{thm:sf-not-poly-comp}, which merely states that there exists a sequence for which $\Sf$ is not computable in polynomial time, and does not assert that $\Sf$ must be computable in superpolynomial time for every sequence.


\section{Discussion}
\label{sec:discussion}

In this paper, we have shown for the first time a loss bound on prediction based on a speed prior. This was proved for $S_{Kt}$, and we suspect that the result for stochastic sequences is not true for $\Sf$, due to weaker bounds on its KL divergence with the true environment. However, in the special case of deterministic sequences, we show that $\Sf$ has the same performance as $S_{Kt}$. We have also, again for the first time, investigated the efficiency of computing speed priors. This offers both encouraging and discouraging news: $S_{Kt}$ is good at prediction in certain environments, but is not efficiently computable, even in the restricted class of environments where it succeeds at prediction. On the other hand, $\Sf$ is efficiently computable for certain inputs, and succeeds at predicting those sequences, but we have no evidence that it succeeds at prediction in the more general case of stochastic sequences.

To illustrate the appeal of speed-prior based inference, it is useful to contrast with a similar approach introduced by \cite{vovk1989prediction}. This approach aims to predict certain simple measures: if $\alpha$ and $\gamma$ are functions $\N \rightarrow \N$, then a measure $\nu$ is said to be $(\alpha, \gamma)$-simple if there exists some `program' $\pi^\nu \in \B^\infty$ such that the UTM with input $x$ outputs $\nu(x)$ in time $\leq \gamma(|x|)$ by reading only $\alpha(|x|)$ bits of $\pi^\nu$. Vovk proves that if $\alpha$ is logarithmic and $\gamma$ is polynomial, and if both $\alpha$ and $\gamma$ are computable in polynomial time, then there exists a measure $\mu_{\alpha, \gamma}$ which is computable in polynomial time that predicts sequences drawn from $(\alpha, \gamma)$-simple measures.

$S_{Kt}$ and $\mu_{\alpha, \gamma}$ are similar in spirit, in that they predict measures that are easy to compute.
However, the contrast between the two is instructive: $\mu_{\alpha, \gamma}$ requires one to fix $\alpha$ and $\gamma$ in advance, and only succeeds on $(\alpha, \gamma)$-simple measures. Therefore, there are many polynomials $\gamma' > \gamma$ such that $\mu_{\alpha, \gamma}$ cannot predict $(\alpha, \gamma')$-simple measures. We are therefore required to make an arbitrary choice of parameters at the start and are limited by that choice of parameters. In contrast, $S_{Kt}$ predicts all measures estimable in polynomial time, and does not require some polynomial to be fixed beforehand. $S_{Kt}$-based prediction therefore is more general than that of $\mu_{\alpha, \gamma}$.

Further questions remain to be studied. In particular, we do not know whether the loss bounds on speed-prior-based predictors can be improved. We also do not know how to tighten the gap between the lower and upper complexity bounds on the speed priors.

It would also be interesting to generalise the definition of $S_{Kt}$. Our performance result was due to the fact that for all measures $\mu$ estimable in polynomial time, $S_{Kt}(x) \geq \mu(x)/(f(|x|, - \log \mu(x)))$, where $f$ was a polynomial. Now, if $\mu$ is estimable in polynomial time by $\nu$, then the denominator of the fraction $\nu(x)$ must be small enough to be printed in polynomial time. This gives an exponential bound on $1/\nu(x)$, and therefore a polynomial bound on $- \log \mu(x)$. We therefore have that $S_{Kt}(x) \geq \mu(x)/g(|x|)$ for a polynomial $g$. Because $g$ is subexponential, this guarantees that $S_{Kt}$ converges to $\mu$~\citep{ryabko2008predicting}.\footnote{To see that $g$ must be subexponential for good predictive results, note that for all measures $\mu$, $\lambda(x) \geq \mu(x)/2^{|x|}$, but $\lambda$ does not predict well.} This suggests a generalisation of $S_{Kt}$ that takes a mixture over some class of measures, each measure discounted by its computation time. Loss bounds can be shown in the same manner as in this paper if the measures are computable in polynomial time, but the question of the computational complexity of this mixture remains completely open.

\subsubsection*{Acknowledgements}

The authors would like to thank the reviewers for this paper, the Machine Intelligence Research Institute for funding a workshop on Schmidhuber's speed prior, which introduced the first author to the concept, and Mayank Daswani and Tom Everitt for valuable discussions of the material. This work was in part supported by ARC grant DP150104590.

\bibliography{bibliography}

\end{document}